\documentclass{article}

\usepackage[preprint]{neurips_2026}

\usepackage{amssymb}
\usepackage{amsmath}
\usepackage[ruled,vlined,linesnumbered]{algorithm2e}
\usepackage{booktabs}
\usepackage{multicol}
\usepackage{multirow}
\usepackage{placeins}
\usepackage{subcaption}
\usepackage{url}
\usepackage{hyperref}
\usepackage{graphicx}
\usepackage{amsthm}
\newtheorem{proposition}{Proposition}
\newtheorem{definition}{Definition}

\title{Soft-MSM: Differentiable Context-Aware Elastic Alignment for Time Series}

\author{
  Christopher Holder \\
  School of Electronics and Computer Science \\
  University of Southampton \\
  Southampton, SO17 1BJ, United Kingdom \\
  \texttt{c.holder@soton.ac.uk}
  \And
  Anthony Bagnall \\
  School of Electronics and Computer Science \\
  University of Southampton \\
  Southampton, SO17 1BJ, United Kingdom \\
  \texttt{anthony.bagnall@soton.ac.uk}
}

\begin{document}

\maketitle

\begin{abstract}
Elastic distances like dynamic time warping (DTW) are central to time series machine learning because they compare sequences under local temporal misalignment. Soft-DTW is an adaptation of DTW that
can be used as a gradient-based loss by replacing the hard minimum in its
dynamic-programming recursion with a smooth relaxation. However, this approach
does not directly extend to elastic distances whose transition costs depend on
the local alignment context. Move-Split-Merge (MSM) is one such distance: it
uses context-aware split and merge penalties and has often outperformed DTW in
supervised and unsupervised time series machine learning tasks such as classification and clustering.

We introduce Soft-MSM, a smooth relaxation of MSM and an elastic
alignment loss with context-aware transition costs. Central to the formulation is a smooth gated surrogate for MSM's piecewise
split/merge cost, which enables gradients through both the dynamic-programming
recursion and the local transition structure. We derive the forward recursion, backward recursion, soft
alignment matrix, closed-form gradient, limiting behaviour, and
divergence-corrected formulation. Experiments on 112 UCR datasets show that Soft-MSM gives lower MSM barycentre loss than existing MSM barycentre methods, and yields significantly better clustering and nearest-centroid classification performance than Soft-DTW-based alternatives. An implementation is available in the
open-source \texttt{aeon} toolkit.
\end{abstract}


\section{Introduction}
Measuring the similarity or distance between objects plays a fundamental role in tasks such as classification, clustering and regression. Time series machine learning (TSML), where we consider a time series as any ordered sequence of real-valued variables, is an active subfield in which distance functions are core components for tasks such as motif discovery, similarity search, and profiling. There are several comprehensive surveys of the extensive research into using distance functions for TSML~\cite{abanda19distance,holder24review,shifaz23elastic}.

Traditional distance functions (e.g., Minkowski) can assign large dissimilarity values to time series that are conceptually similar but slightly misaligned. To address this, a key focus in TSML research has been the development of time series specific distance functions that account for temporal misalignment. These are often called elastic distances, since they form an alignment path that conceptually stretches or warps series onto each other. Elastic distances have been found to be more effective for TSML than traditional distances for tasks such as classification~\cite{lines14elastic} and clustering~\cite{holder24review}.

Dynamic time warping (DTW)~\cite{berndt94dtw} is by far the most popular elastic distance measure for time series. It forms an alignment path between two series that minimises the pointwise distance between the series using dynamic programming with the Bellman recursion~\cite{bellman59adaptivecontrol}. Figure~\ref{fig:alignment-path-visualisation} characterises the benefits of realignment: Euclidean distance measures the series as being quite different even though the pattern of peaks and troughs is similar. DTW better captures this similarity and gives a smaller measure of distance than Euclidean distance.

\begin{figure}[h]
    \centering
    \begin{subfigure}[b]{0.45\textwidth}
        \centering
        \includegraphics[width=\textwidth]{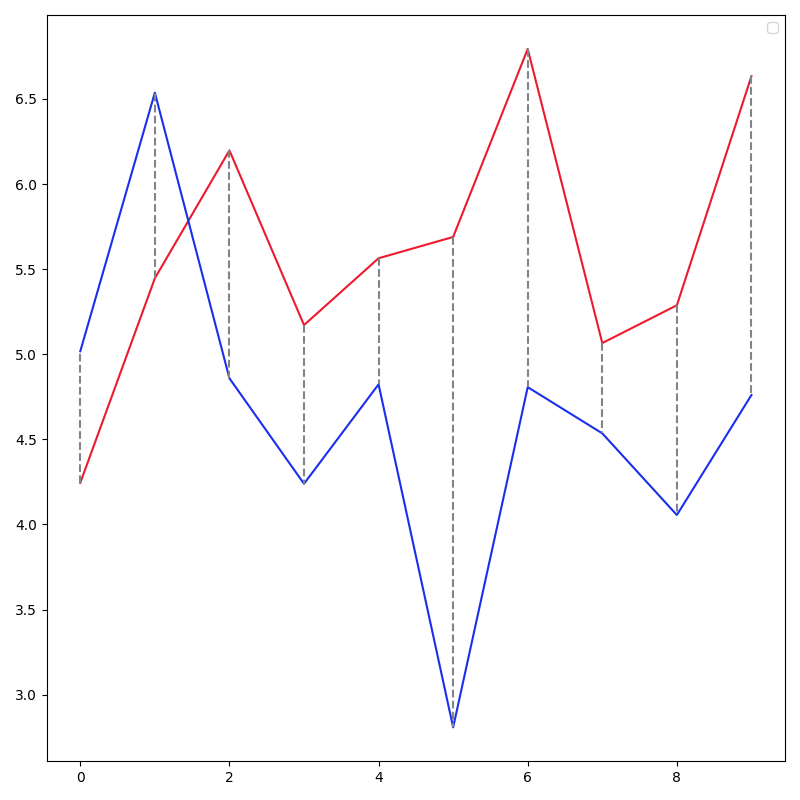}
        \caption{Euclidean distance.}
    \end{subfigure}
    \hfill
    \begin{subfigure}[b]{0.45\textwidth}
        \centering
        \includegraphics[width=\textwidth]{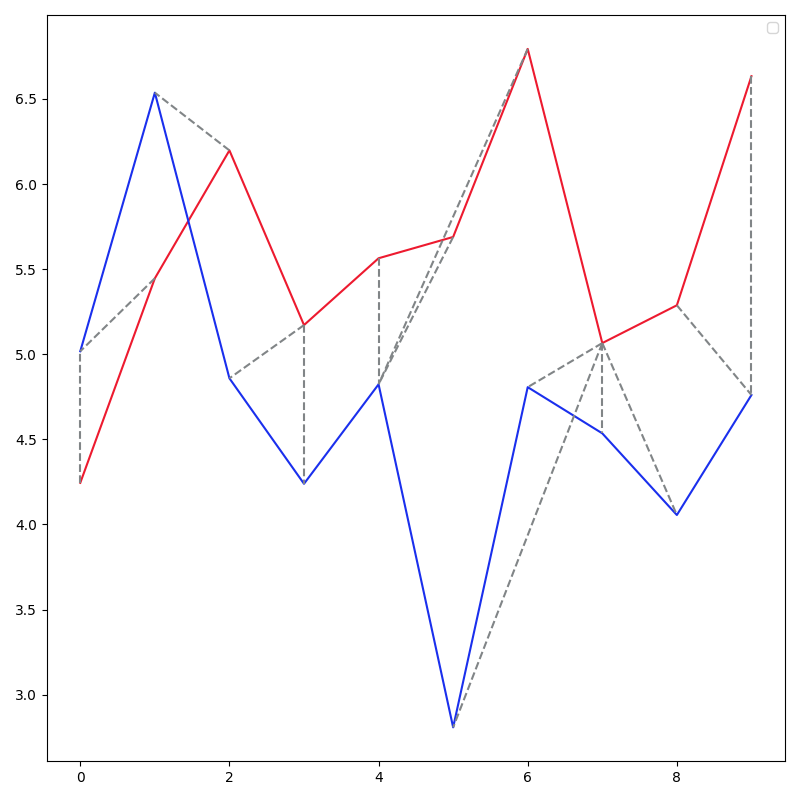}
        \caption{DTW distance.}
    \end{subfigure}
    \caption{Visualisation of optimal alignment paths between two time series for the Euclidean and DTW distance. The dashed grey lines represent the alignment of points from the red time series to the blue time series.}
    \label{fig:alignment-path-visualisation}
\end{figure}

DTW has been used successfully in many applications and has a significant impact on the TSML field. Nevertheless, it has limitations. Firstly, it is not a metric, meaning the triangle inequality does not hold. This limits the potential for acceleration in tasks such as clustering~\cite{elkan03elkankmeans}. Secondly, DTW does not impose an explicit penalty for deviating from the diagonal. This issue is often mitigated by using a fixed warping window to prevent pathologically large warping. However, it has been shown that DTW can still vacillate within the window, leading to degraded performance in tasks such as clustering~\cite{holder24review}. These limitations have been addressed in various ways by introducing penalties for warping: Amerced DTW~\cite{herrmann23adtw} adds a constant penalty for off-diagonal movement, while Move-Split-Merge (MSM)~\cite{stefan13msm} and Time Warp Edit (TWE)~\cite{marteau09twe} are both metrics that explicitly penalise warping. These enhanced elastic distances yield improved performance in tasks such as clustering and serve as primitives for more sophisticated distance-based classification algorithms, such as Proximity Forest~\cite{tan25pf2}.

However, a limitation of all elastic distances is that they are non-differentiable with respect to the time series arguments. This means that small changes in the time series can cause large changes in the optimal warping path. The inability to find a gradient function means that elastic distances cannot be directly used as loss functions in gradient-based machine learning algorithms.

This limitation was addressed for DTW through the introduction of Soft-DTW~\cite{cuturi2017softdtw}.
Soft-DTW modifies the traditional DTW formulation by smoothing the cost function: it replaces the hard minimum in the Bellman recursion with a differentiable soft-minimum operator (defined in Section~\ref{sec:soft-dtw}).
Instead of selecting a single optimal alignment path, Soft-DTW computes a weighted average over all possible paths, assigning exponentially higher weights to those with lower costs.


A soft elastic distance function has a range of possible applications. For example, in neural networks that generate time series data, the gradient vector provides information about how to adjust the network's output to better match target sequences. The alignment matrix is also useful for time series averaging, where the gradient information guides the iterative refinement of a centroid sequence. Additionally, in time series classification tasks with attention mechanisms, the gradient vector can help identify which parts of the sequences are most important for classification decisions.

These observations motivate Soft-MSM: elastic distances other than DTW have been shown to be useful, and the ability to make a gradient matrix for DTW enhances its utility. Making distance functions differentiable increases their scope of application. The Soft-DTW adaptations are specific to DTW and are not immediately adaptable to other elastic distances. Our goal is to take one of the best performing elastic distance measures, MSM, and make it differentiable.

We propose Soft-MSM, a differentiable version of Move-Split-Merge. The main difficulty is the MSM split/merge cost, which depends on the local ordering of three values and is therefore piecewise. We replace this cost with a smooth gated transition function, so that gradients can be propagated through both the dynamic-programming recursion and the context-dependent transition costs. We derive the forward and backward recursions, the resulting soft alignment matrix, and the gradient with respect to the input series. We also analyse differentiability, limiting behaviour, metricity, and a divergence-corrected form of the objective. Experiments on 112 UCR datasets show that Soft-MSM gives lower MSM barycentre loss than Soft-DTW, with corresponding improvements in clustering and nearest-centroid classification. The method is implemented in \texttt{aeon}.


The remainder of this paper is organised as follows. Section~\ref{sec:background} introduces the relevant notation and definitions, and outlines key time series distance functions such as DTW, MSM, and Soft-DTW. It also provides background on time series averaging, which serves as one of the main evaluation benchmarks in this study. Section~\ref{sec:soft-msm} presents the novel differentiable formulation of MSM, describing both the forward and backward recursions used to derive the Soft-MSM alignment matrix and Jacobian transformation. Section~\ref{sec:results} outlines the experimental setup and reports results across multiple learning scenarios. Finally, Section~\ref{sec:conclusions} summarises the main findings and discusses directions for future research.

\section{Background}
\label{sec:background}

A time series is an ordered sequence of \(m\) real-valued observations of a variable, denoted \(\mathbf{x} = (x_1, \dots, x_m)\).
If each observation \(x_i\) is a vector, then \(\mathbf{x}\) is a multivariate time series. The learning tasks we explore involve a collection of $n$ time series, \(\mathcal{X} = \{\mathbf{x}^{(1)}, \mathbf{x}^{(2)}, \dots, \mathbf{x}^{(n)}\}\).  Our experiments are restricted to the case where each \(x_i\) is a scalar and all time series have equal length \(m\). Each individual series \(\mathbf{x}^{(i)} = (x^{(i)}_1, x^{(i)}_2, \dots, x^{(i)}_m)\) represents the \(i\)-th element of the dataset \(\mathcal{X}\). These assumptions are adopted to streamline notation and constrain confounding factors in experimentation. In practice, all distance measures and algorithms discussed here can be readily extended to multivariate and unequal-length series~\citep{shifaz23elastic}.
Moreover, the implementation accompanying this paper supports both multivariate and unequal-length time series~\cite{aeon24jmlr}.

\subsection{Time series distance measures}

A pointwise distance function $\delta(x,y): \mathbb{R} \times \mathbb{R} \to \mathbb{R}$ measures the distance between two scalars.
A pointwise distance matrix between two series, $D(\mathbf{x},\mathbf{y})$, is simply $D_{i,j}= \delta(x_i,y_j)$.
A time series distance function $d$ is a mapping from the domain of the Cartesian product of two $\mathbb{R}^m$ spaces to the codomain of real numbers $\mathbb{R}$,
$d(\mathbf{x},\mathbf{y}): \mathbb{R}^m \times \mathbb{R}^m \to \mathbb{R}$. Distance functions quantify the dissimilarity between two series.




Traditional distance measures, such as Minkowski distances, assume that two time series are perfectly aligned in time.
This assumption often fails in practice, as there may be global or local misalignment between series. To address this limitation, elastic distances compensate for temporal misalignments by allowing local stretching and compression along the time axis. An elastic distance operates by finding an alignment path, a sequence of index pairs specifying which elements of \(\mathbf{x}\) and \(\mathbf{y}\) are aligned:
\[
P = \langle (e_1, f_1), (e_2, f_2), \ldots, (e_s, f_s) \rangle .
\]
The path satisfies the constraints of clamped start and end points,
\[
(e_1, f_1) = (1, 1), \qquad (e_s, f_s) = (m, m),
\]
and monotonic progression
\[
0 \leq e_{t+1} - e_t \leq 1, \qquad
0 \leq f_{t+1} - f_t \leq 1, \qquad
\forall\, t < s.
\]
This alignment path allows points in one series to be matched to non-synchronous points in the other, thereby accommodating local time distortions.
For a given alignment path $P$ under pointwise local costs, let
\begin{equation}
\label{eq:path-cost}
\mathcal{C}(P;x,y)
=
\sum_{(i,j)\in P} \delta(x_i,y_j)
\end{equation}
denote the total cost accumulated along that path. The alignment path that minimises the total accumulated distance between the two time series is denoted by \(P^{*}\).

An alignment path can be represented either as a list of index pairs or equivalently as an alignment matrix \(A \in \mathbb{R}^{m \times m}\).
In the hard (discrete) case, \(A_{i,j} = 1\) if \((i,j) \in P\) and \(0\) otherwise.
Relaxing this binary constraint allows for soft or stochastic alignments, where \(A_{i,j}\) expresses the relative weight or probability that points \(x_i\) and \(y_j\) are aligned.
Figure~\ref{fig:alignment-matrix-example} illustrates a binary alignment matrix and path.

\begin{figure}[htb]
    \centering
    \includegraphics[width=1.0\textwidth]{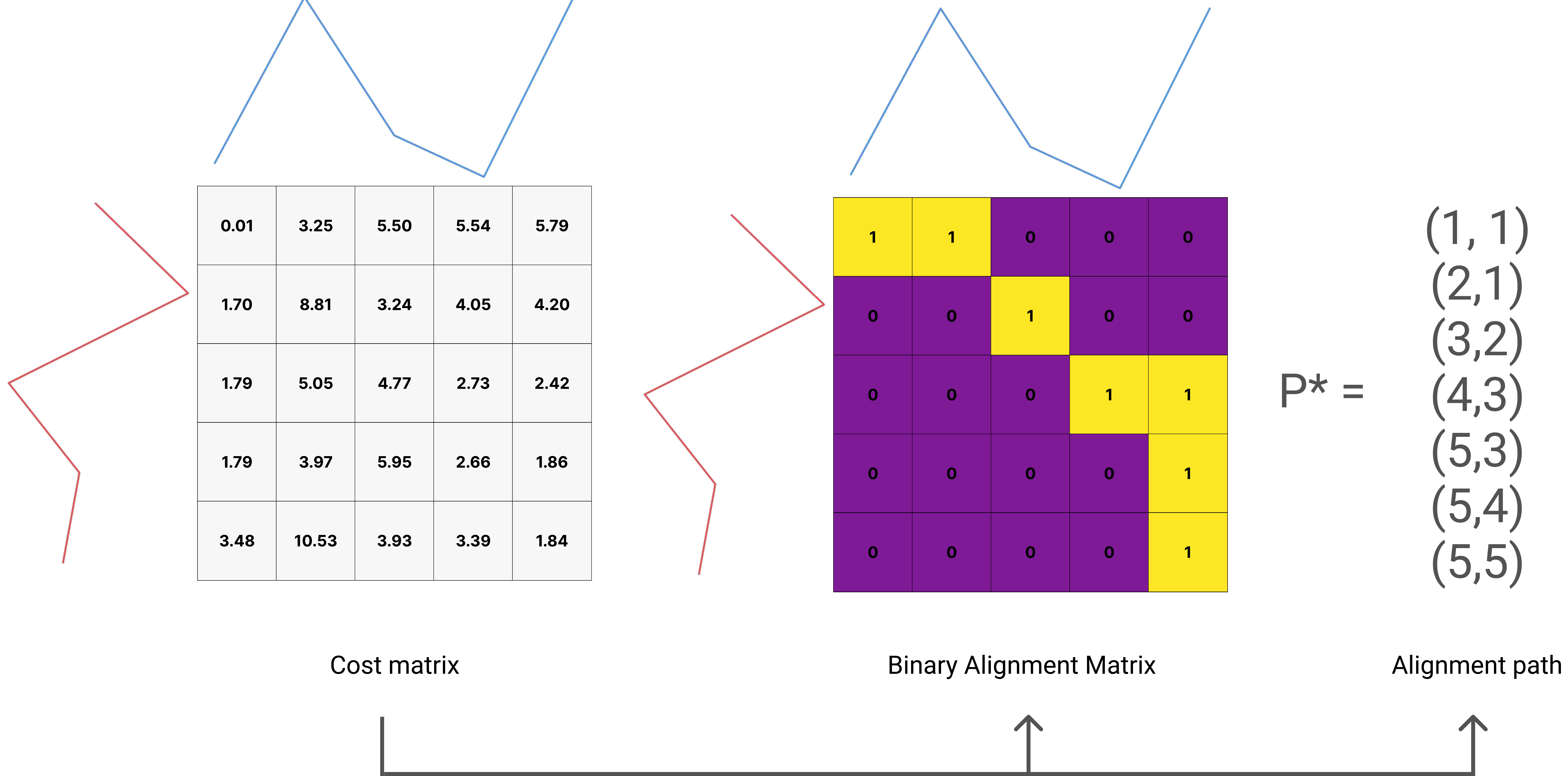}
    \caption{Two different representations of the optimal alignment path through a cost matrix.}
    \label{fig:alignment-matrix-example}
\end{figure}

In practice, elastic distances are computed not by directly enumerating paths in \(D\), but through a recursive optimisation over an accumulated cost matrix \(C \in \mathbb{R}^{m \times m}\).
Each entry \(C_{i,j}\) represents the cost of the best alignment path or, in soft variants, a smooth aggregation over paths, that reaches index \((i,j)\).
This recursion integrates both the local distance \(D_{i,j}\) and a penalty for deviating from the diagonal, which discourages excessive warping.
The total alignment cost between the two series is then obtained from the final index, \(C_{m,m}\).
Different elastic distance measures vary in how this recursion and off-diagonal penalty are defined, but they all share this general dynamic programming formulation.

\subsection{Dynamic Time Warping (DTW)~\cite{berndt94dtw}}

Dynamic Time Warping (DTW) is the best known elastic distance for time series machine learning~\citep{berndt94dtw} and has been employed in thousands of publications. DTW uses dynamic programming to identify the optimal alignment path through a cost matrix \(C\) that minimises the cumulative distance between two time series.
To begin, we initialise a cost matrix $C \in \mathbb{R}^{m \times m}$ with
$C_{1,1} = (x_{1} - y_{1})^{2}$. The boundary conditions are then defined as:
\[
C_{i,1} = (x_{i} - y_{1})^{2} + C_{i-1,1}, \qquad
C_{1,j} = (x_{1} - y_{j})^{2} + C_{1,j-1}.
\]
and then the cost matrix is found recursively for $i\ge 2$ and $j\ge 2$:
\begin{equation}
    C_{i,j} = (x_i - y_j)^2 + \min \begin{cases}
        C_{i-1,\; j-1} \\
        C_{i-1,\; j} \\
        C_{i,\;j-1}
    \end{cases}
    \label{eq:dtw-cm}
\end{equation}

The DTW distance is then the last value in the cost matrix $C_{m,m}$ and the alignment path can be found by backtracking through $C$. DTW does not explicitly impose a penalty for deviating from the diagonal; instead, it accumulates an implicit penalty by increasing the total path length. DTW is not a metric nor is it differentiable.

\subsection{Move-Split-Merge (MSM)~\cite{stefan13msm}}
\label{sec:msm}

The Move-Split-Merge (MSM) distance is an elastic measure that extends the dynamic programming framework of DTW by introducing context-dependent transition costs.
While DTW accumulates only pointwise distances, MSM adds an explicit penalty for transitions that move off the diagonal, thereby modelling insertions and deletions in a more interpretable way.

MSM differs from DTW in two key aspects.
First, it replaces the squared difference with the absolute difference between points.
Second, and more importantly, it introduces a context-aware transition cost function, $\mathrm{msm\_cost}$, defined in Equation~\ref{eq:msm-cost}.
This function governs the cost of off-diagonal transitions and captures how the alignment between neighbouring points should behave.

The $\mathrm{msm\_cost}$ function quantifies how expensive it is to add or remove a value \(x\) between its neighbours \(y\) and \(z\).
If \(x\) lies between \(y\) and \(z\), a fixed penalty \(c\) is applied.
Otherwise, an additional penalty proportional to the deviation is added. This transition cost function is defined as:
\begin{equation}
\mathrm{msm\_cost}(x,y,z) =
\begin{cases}
c, & \text{if } y \leq x \leq z \text{ or } y \geq x \geq z, \\
c + \min(|x - y|, |x - z|), & \text{otherwise.}
\end{cases}
\label{eq:msm-cost}
\end{equation}

To compute the MSM distance, we first initialise a cost matrix $C \in \mathbb{R}^{m \times m}$ with
$C_{1,1} = |x_{1} - y_{1}|$. The boundary conditions are then defined as:
\[
C_{i,1} = \mathrm{msm\_cost}(x_i, x_{i-1}, y_1) + C_{i-1,1}, \qquad
C_{1,j} = \mathrm{msm\_cost}(y_j, y_{j-1}, x_1) + C_{1,j-1}.
\]
and then the MSM cost matrix is found recursively for $i\ge 2$ and $j\ge 2$:
\begin{equation}
C_{i,j} = \min\begin{cases}
|x_i - y_j| + C_{i-1,\;j-1}\\
\mathrm{msm\_cost}(x_i, x_{i-1}, y_j) + C_{i-1,\;j} \\
\mathrm{msm\_cost}(y_j, y_{j-1}, x_i) + C_{i,\;j-1}
\end{cases}
\label{eq:msm-cm}
\end{equation}
The MSM distance is then the last value in the cost matrix $C_{m,m}$. The first term corresponds to a diagonal ``match'' operation, while the second and third represent vertical and horizontal transitions corresponding to ``split'' and ``merge'' operations, respectively. MSM has been shown to perform effectively across a wide range of time series learning tasks~\cite{holder24review, holder23mba, lines15elastic, bagnall17bakeoff}.
It was also proven to satisfy the properties of a metric~\cite{stefan13msm}, although, like DTW, it remains non-differentiable.

\subsection{Soft-DTW~\cite{cuturi2017softdtw}}
\label{sec:soft-dtw}

Soft-DTW~\cite{cuturi2017softdtw} is a reformulation of DTW obtained by smoothing its Bellman recursion, making the objective differentiable for any temperature parameter $\gamma>0$.
This modification enables its use as a loss function in gradient-based learning while preserving the alignment-based geometry of DTW. To achieve this, the hard minimum in the DTW recursion is replaced with a smooth minimum operator defined as:
\begin{equation}
\label{eq:softmin}
\operatorname{softmin}_\gamma(a_1,\dots,a_k)
  = -\,\gamma \log \!\left( \sum_{r=1}^{k} e^{-(a_r - s)/\gamma} \right)
    + s,
  \qquad s = \min_{1 \le r \le k} a_r, \;\; \gamma > 0.
\end{equation}

$s$ is a stabilising constant used for numerical stability.  This converges to $\min\{a_1,\ldots,a_k\}$ as $\gamma \to 0^{+}$.
By replacing the hard minimum with the soft minimum operator, the DTW forward recursion in Equation~\ref{eq:dtw-cm} becomes
\begin{equation}
\label{eq:soft-dtw-cm-clean}
C_{i,j}
=
D_{i,j}
+
\operatorname{softmin}_\gamma\!\begin{cases}
        C_{i-1,\;j-1}\\
        C_{i-1,\;j}\\
        C_{i,\;j-1}
    \end{cases}
\end{equation}
As $\gamma \rightarrow 0$, Soft-DTW becomes the DTW recursion. To find the gradient of Soft-DTW with respect to changes in the values of the input series, $\mathbf{x} = (x_1,\ldots,x_m)$, the first step is to calculate the alignment matrix  $A\in\mathbb{R}^{m \times m}$, where $A_{i,j}$ expresses the relative weight of aligning $x_i$ and $y_j$ over all possible paths. The matrix $A$ is found through a backward recursion. It is initialised such that $A_{m,m}=1$ and $A_{i,j}=0$ elsewhere, then the weight of cell $(i,j)$ is based on weights implied by the forward cost function,
\[
\begin{aligned}
w_{i-1,j-1} &= \exp\!\Big(\tfrac{\operatorname{softmin}_\gamma(C_{i-1,j-1},C_{i-1,j},C_{i,j-1}) - C_{i-1,j-1}}{\gamma}\Big),\\
w_{i-1,j}   &= \exp\!\Big(\tfrac{\operatorname{softmin}_\gamma(C_{i-1,j-1},C_{i-1,j},C_{i,j-1}) - C_{i-1,j}}{\gamma}\Big),\\
w_{i,  j-1} &= \exp\!\Big(\tfrac{\operatorname{softmin}_\gamma(C_{i-1,j-1},C_{i-1,j},C_{i,j-1}) - C_{i,  j-1}}{\gamma}\Big).
\end{aligned}
\]
and the alignment matrix is updated as
\[
A_{i-1,j-1}\mathrel{+}=A_{i,j}\,w_{i-1,j-1},\qquad
A_{i-1,j}\mathrel{+}=A_{i,j}\,w_{i-1,j},\qquad
A_{i,j-1}\mathrel{+}=A_{i,j}\,w_{i,j-1}.
\]
The resulting matrix $A$ provides a soft alignment between the two time series. Indices along low-cost routes receive higher weights, and as $\gamma\to 0^{+}$ these weights collapse onto a single hard path corresponding to the standard DTW alignment.

While the alignment matrix is useful for analysing path contributions, it also represents the gradient of the Soft-DTW objective with respect to the pointwise distance matrix $D$, in this case the squared Euclidean distance.  For most downstream applications, such as time series averaging, clustering, or model training, the goal is to obtain the gradient with respect to the time series itself.
This can be achieved by performing a Jacobian decomposition of the alignment matrix.

Let $J=\partial D/\partial x$ denote the Jacobian of the local costs with respect to $x$. For Euclidean distance,
\[
\frac{\partial D_{i,j}}{\partial x_k}
=
\begin{cases}
2(x_i-y_j), & k=i,\\
0, & k\ne i.
\end{cases}
\]
Since $\mathrm{\text{Soft-DTW}}_\gamma(\mathbf{x},\mathbf{y})=C_{m,m}$, the gradient with respect to $\mathbf{x}$ is
\[
\nabla_{\!\mathbf{x}}\,\mathrm{\text{Soft-DTW}}_\gamma(\mathbf{x},\mathbf{y})
=
J^\top A,
\qquad
\frac{\partial\,\mathrm{\text{Soft-DTW}}_\gamma(\mathbf{x},\mathbf{y})}{\partial x_i}
=
\sum_{j=1}^m 2(x_i-y_j)\,A_{i,j}.
\]
This closed-form expression enables direct use of Soft-DTW in gradient-based optimisation for tasks such as averaging, clustering, and parameter learning.

Although not the primary focus of this paper, recent work~\cite{blondel21softdtwdivergence} has highlighted several undesirable properties of Soft-DTW when used as a loss function, and proposed a remedy in the form of a divergence formulation. In its original form, Soft-DTW may yield negative values under certain local cost functions and does not necessarily evaluate to zero when $x = y$. These characteristics are undesirable for a loss function, as they compromise its interpretability as a dissimilarity measure.  The soft-DTW divergence~\cite{blondel21softdtwdivergence} addresses these issues by removing the self-similarity bias:
\[
D_\gamma(x,y)
=
\mathrm{\text{Soft-DTW}}_\gamma(x,y)
-\tfrac{1}{2}\,\mathrm{\text{Soft-DTW}}_\gamma(x,x)
-\tfrac{1}{2}\,\mathrm{\text{Soft-DTW}}_\gamma(y,y).
\]
Soft-DTW divergence is non-negative and equals zero when $x=y$~\cite{blondel21softdtwdivergence}.


\subsection{Time series averaging}
\label{sec:time-series-averaging}

The objective of time series averaging is to construct a representative time series that lies at the centre of a given set of time series~\cite{schultz18ssgba}.
The simplest method is to compute the arithmetic mean at each time point, which minimises the sum of squared Euclidean distances to all series in the collection.

However, like Minkowski distances, the arithmetic mean assumes that the time series are perfectly aligned.
It therefore fails to account for temporal misalignments, leading to a poor average.
Given the success of elastic distances in aligning time series, several methods have been proposed to incorporate alignment directly into the averaging process.

\subsubsection{Elastic Barycentre Averaging}

To integrate temporal alignment into averaging, the problem can be reformulated as an optimisation task:
finding the sequence $\boldsymbol{\beta}$ that minimises the Fr\'echet function~\citep{frechet48frechetfunction} under a chosen distance measure $d$.
Formally,
\begin{equation}
\label{eq:frechet}
F(\boldsymbol{\beta}) \;=\; \frac{1}{n} \sum_{i=1}^{n} d\big(\boldsymbol{\beta}, \mathbf{x}^{(i)}\big)^{2},
\end{equation}
where $\boldsymbol{\beta}$ is the candidate barycentre, and $d$ is a time series distance function.
The goal is then to find the optimal barycentre
\begin{equation}
\label{eq:frechet-argmin}
\boldsymbol{\beta}^{\ast} \;=\; \arg\min_{\boldsymbol{\beta}} F(\boldsymbol{\beta}).
\end{equation}
This barycentre generalises the notion of an arithmetic mean to elastic distance spaces, providing a representative time series that aligns with temporal variability across the dataset.

The first method to successfully incorporate alignment into the averaging process was DTW Barycentre Averaging (DBA)~\cite{petitjean11dba}.
DBA minimises the Fréchet function using an iterative heuristic based on pairwise DTW alignments.
Starting from an initial average (typically the arithmetic mean), each iteration aligns all series in the collection to the current average, collects values associated with each aligned position, and updates the average by taking the arithmetic mean of these aligned values.

In the original formulation by~\cite{petitjean11dba}, the Fréchet function was minimised specifically for DTW.
This idea was later generalised in the Elastic Barycentre Average (EBA)~\cite{holder23mba}, which relaxed the restriction to DTW and allowed the use of any elastic distance that computes a complete alignment path.
Employing alternatives to DTW has been shown to yield improved performance across multiple learning tasks.
For example, using the MSM and Shape-DTW~\cite{zhao18shapedtw} distances led to the MSM Barycentre Average (MBA)~\cite{holder23mba} and Shape-DTW Barycentre Average (Shape-DBA)~\cite{ali23shapedba}, both of which significantly outperform DBA and achieve state-of-the-art performance for clustering.

While effective, DBA and its extensions are computationally expensive, often requiring many refinement iterations to converge.
To address this,~\cite{schultz18ssgba} proposed the Stochastic Subgradient Dynamic Time Warping Barycentre Average (SSG-DBA), a subgradient-based optimisation strategy that achieves similar-quality results while substantially reducing runtime.

A further approach to minimising the Fréchet function is to use a gradient-based optimisation method.
However, this requires both the distance measure and the associated cost matrix to be differentiable.
Using Soft-DTW,~\cite{cuturi2017softdtw} introduced Soft-DBA, which integrates all possible alignment paths through a smooth differentiable formulation.

Soft-DBA computes the Soft-DTW distance and its Jacobian (as outlined in Section~\ref{sec:soft-dtw}) between a candidate average and each time series in the collection.
The resulting gradients are then used with the Limited-memory Broyden–Fletcher–Goldfarb–Shanno with Bound constraints (L-BFGS-B) optimisation algorithm~\cite{zhu97algorithm}.
L-BFGS-B begins from an initial average time series and iteratively refines it using gradient information to minimise the overall Soft-DTW loss across all series.
At each iteration, it determines an update direction and step size that jointly reduce the total distance.
The process continues until convergence, typically when the change in total Soft-DTW loss between consecutive iterations falls below a defined tolerance threshold.
The final average time series is then returned as the Soft-DTW barycentre.

When compared to DBA and SSG-DBA, Soft-DBA significantly outperforms both in averaging and clustering tasks~\cite{holder24tsclwithkmeans}. This improvement is likely attributable to Soft-DTW’s ability to compute an exact gradient, whereas DBA and SSG-DBA rely on gradient estimates. However, despite its advantages over other DBA variants, Soft-DBA still performs worse than alternative elastic measures such as MBA, and Shape-DBA for the same tasks~\cite{holder24kasba}. We therefore hypothesise that if these other elastic measures could be made differentiable, substantial performance gains could be achieved.

\section{Soft-MSM}
\label{sec:soft-msm}

Before presenting the Soft-MSM recursion, Table~\ref{tab:soft-msm-notation} summarises the main notation used throughout this section.
\begin{table}[!htb]
\centering
\caption{Core notation used in the Soft-MSM formulation.}
\label{tab:soft-msm-notation}
\begin{tabular}{ll}
\toprule
Symbol & Meaning \\
\midrule
\(x, y\) & Input time series of length \(m\) \\
\(c\) & MSM split/merge penalty parameter \\
\(\gamma\) & Soft-minimum temperature parameter \\
\(\epsilon\) & Numerical stabiliser used in the smooth gate \\
\(C\) & Accumulated Soft-MSM cost matrix \\
\(F_\gamma(x,y)\) & Soft-MSM objective, equal to \(C_{m,m}\) \\
\(A\) & Soft alignment matrix from the backward recursion \\
\(g(x,y,z)\) & Smooth between-gate used in the transition cost \\
\(\operatorname{trans}_\gamma(x,y,z;c)\) & Differentiable MSM transition cost \\
\(G_{i,j}\) & Weight associated with a diagonal match edge \\
\(V_{i,j}\) & Weight associated with a vertical transition edge \\
\(H_{i,j}\) & Weight associated with a horizontal transition edge \\
\bottomrule
\end{tabular}
\end{table}

Soft-MSM is a differentiable alignment-based loss that preserves the MSM property of an explicit penalty for off-diagonal moves. DTW was made differentiable by replacing the hard minimum in the Bellman recursion with a smooth minimum operator~\citep{cuturi2017softdtw}. It is more complex for MSM, since the MSM local transition cost in Equation~\ref{eq:msm-cost} is piecewise, with switching conditions \(y \le x \le z\) or \(y \ge x \ge z\). Crossing these boundaries produces non-smooth regions in the cost matrix, resulting in undefined gradients at transition points. There are three sources of non-differentiability that we address with Soft-MSM:\\

\noindent \textbf{1. Differentiable local distance}

MSM employs absolute differences (Equation~\ref{eq:msm-cost}), whose derivative is undefined at \(x = y\) since the one-sided derivatives are \(-1\) and \(+1\). Consequently, the overall dynamic-programming objective is non-differentiable whenever any aligned pair coincides.
We replace the absolute distance with the squared distance because it is continuously differentiable and provides smooth gradients everywhere.\\

\noindent \textbf{2. Differentiable minimum in $C$}

The cost function hard minimum in the MSM recursion (Equation~\ref{eq:msm-cm}) can be replaced with the soft minimum operator (Equation~\ref{eq:softmin}) in a manner identical to Soft-DTW.\\

\noindent \textbf{3. Differentiable cost function}
\label{sec:soft-msm-cost}

A central component of Soft-MSM is a smooth reformulation of the MSM
transition cost. The original MSM cost in Equation~\ref{eq:msm-cost} is
piecewise: it applies a fixed penalty when \(x\) lies between \(y\) and
\(z\), and otherwise adds the smaller deviation from \(y\) or \(z\). We
replace this piecewise rule with a differentiable surrogate.

Let \(a=x-y\), \(b=x-z\), and \(u=ab\). We define the smooth between-gate
\begin{equation}
\label{eq:between-gate}
g(x,y,z)
=
\tfrac{1}{2}\left(1-\frac{u}{\sqrt{u^2+\varepsilon}}\right),
\qquad \varepsilon>0.
\end{equation}
When \(x\) lies between \(y\) and \(z\), the quantities \(a\) and \(b\)
have opposite signs, so \(u<0\) and \(g(x,y,z)\approx 1\). Otherwise,
\(u>0\) and \(g(x,y,z)\approx 0\). The boundary cases \(x=y\) or \(x=z\)
give \(u=0\), for which the smooth gate takes the intermediate value
\(g=1/2\). In all experiments and implementations, we fix
\(\varepsilon=10^{-12}\) as a numerical stabiliser.

Using squared deviations, we define the Soft-MSM transition cost as
\begin{equation}
\label{eq:soft-msm-trans}
\mathrm{trans}_\gamma(x,y,z;c)
=
c
+
\bigl(1-g(x,y,z)\bigr)
\mathrm{softmin}_\gamma\bigl((x-y)^2,(x-z)^2\bigr).
\end{equation}
Thus, when \(x\) lies between \(y\) and \(z\), the transition cost is close
to the fixed MSM penalty \(c\). Otherwise, the cost adds a smooth analogue
of \(\min\{(x-y)^2,(x-z)^2\}\). Algorithm~\ref{alg:soft-msm-trans}
summarises the computation used in our implementation.

\begin{algorithm}[htb]
\DontPrintSemicolon
\LinesNumbered
\caption{$\mathrm{trans}_\gamma(x,y,z;c)$: Soft-MSM transition cost with smooth between-gate}
\label{alg:soft-msm-trans}
\KwIn{scalars $x,y,z$; cost $c>0$; temperature $\gamma>0$}
\KwOut{$t \in \mathbb{R}$}
\BlankLine
$\varepsilon \leftarrow 10^{-12}$ \tcp*{fixed numerical stabiliser}

$a \leftarrow x - y$\;
$b \leftarrow x - z$\;
$u \leftarrow a \cdot b$\;

$g \leftarrow \tfrac{1}{2}\!\left(1 - \dfrac{u}{\sqrt{u^2 + \varepsilon}}\right)$\;

$s \leftarrow \mathrm{softmin}_\gamma\!\bigl((x-y)^2,\,(x-z)^2\bigr)$\;

$t \leftarrow c + (1 - g)\cdot s$\;

\Return $t$
\end{algorithm}

To compute the Soft-MSM distance using this function, a cost matrix \(C \in \mathbb{R}^{m \times m}\) is initialised with
\(C_{1,1} = (x_{1} - y_{1})^2\). The boundary conditions are then defined as:
\[
C_{i,1} = \mathrm{trans}_\gamma(x_i, x_{i-1}, y_1;c) + C_{i-1,1}, \qquad
C_{1,j} = \mathrm{trans}_\gamma(y_j, y_{j-1}, x_1;c) + C_{1,j-1}.
\]

The remaining entries of the Soft-MSM cost matrix are computed on the forward pass for \(i \ge 2\) and \(j \ge 2\) according to:
\begin{equation}
C_{i,j} = \mathrm{softmin}_\gamma\!\begin{cases}
(x_i - y_j)^2 + C_{i-1,\;j-1} \\
\mathrm{trans}_\gamma(x_i, x_{i-1}, y_j;c) + C_{i-1,\;j} \\
\mathrm{trans}_\gamma(y_j, y_{j-1}, x_i;c) + C_{i,\;j-1}
\end{cases}
\label{eq:soft-msm-cm}
\end{equation}

 As \(\gamma \to 0^{+}\), the \(\mathrm{softmin}_\gamma\) operator converges
to the pointwise minimum. Away from the boundary cases \(x=y\) and \(x=z\),
as \(\varepsilon \to 0^{+}\), the smooth gate in
Equation~\ref{eq:between-gate} tends to \(1\) when \(x\) lies between \(y\)
and \(z\), and to \(0\) otherwise. Thus Equation~\ref{eq:soft-msm-cm}
approaches a hard-min MSM-style recursion with the same move/split/merge
structure as MSM, but with squared rather than absolute local costs. We
formalise this limiting behaviour and differentiability in
Section~\ref{sec:theory}.

To differentiate Soft-MSM with respect to an input series, we define the
backward pass, which yields both the alignment matrix and the Jacobian.

\subsection{Soft-MSM alignment matrix}
\label{sec:softmsm-alignment}

Finding $A$ is more complex than with Soft-DTW, since the diagonal components contribute different weights than the off-diagonal. The algorithm for the backward recursion to find the alignment values is described in Algorithm~\ref{alg:soft-msm-backward}.

Each value $A_{i,j}$ is a weighted average of the alignment of the three positions that can be reached from $(i,j)$. Weights are the derivative of the softmin operator appearing in the forward recursion. To find these, we find the cost of the move between $(i,j)$ and one of $(i',j') \in \{(i+1,j), (i,j+1), (i+1,j+1)\}$.

In the backward pass we propagate the gradient
\[
A_{i,j} \;=\; \frac{\partial \mathcal{F_\gamma}}{\partial C_{i,j}},
\]
and for each successor $(i', j')$ we use the chain rule to accumulate contributions
\[
A_{i,j} \;\leftarrow\; A_{i,j}
\;+\;
A_{i',j'} \,
\frac{\partial C_{i',j'}}{\partial C_{i,j}}.
\]
The partial derivative is
\[
\frac{\partial C_{i',j'}}{\partial C_{i,j}}
=
\exp\!\left(
\frac{C_{i',j'} - \bigl(C_{i,j} + \tau\bigr)}{\gamma}
\right).
\]
Here $\tau \in \{\tau_d, \tau_h, \tau_v\}$ denotes the transition cost for diagonal, horizontal, and vertical moves, respectively.
For the diagonal move,
\[
\tau_d = (x_{i+1}-y_{j+1})^2.
\]
For the off-diagonal terms $\tau_h$ and $\tau_v$, we recompute the transition costs using the $\mathrm{trans}_\gamma$ function (see lines~9 and~12 of Algorithm~\ref{alg:soft-msm-backward}).
The differential of the soft-min operation then yields the weights:

\begin{align}
\label{eq:softmsm-backward}
A_{i,j}
&=
A_{i+1,j}\,
\exp\!\left(\frac{C_{i+1,j} - \bigl(C_{i,j} + \tau_v\bigr)}{\gamma}\right)
\;+\;
A_{i,j+1}\,
\exp\!\Big(\tfrac{C_{i,j+1} - \bigl(C_{i,j} + \tau_h)}{\gamma}\Big) \notag\\
&\quad+\;
A_{i+1,j+1}\,
\exp\!\Big(\tfrac{C_{i+1,j+1} - \bigl(C_{i,j} + \tau_d)}{\gamma}\Big).
\end{align}

The resulting matrix $A$ provides a soft alignment between the two time series: entries along low-cost routes receive higher weights, and as $\gamma \to 0^{+}$ these weights collapse to a single hard path corresponding to the classical MSM alignment. $A$ represents the derivative of the objective with respect to the accumulated cost.

\begin{algorithm}[H]
\DontPrintSemicolon
\LinesNumbered
\caption{$\mathrm{alignment\_matrix}_\gamma(x,y,c,C)$}
\label{alg:soft-msm-backward}
\KwIn{$x,y$; cost $c>0$; temperature $\gamma>0$; cost matrix $C$ from Eq.~\ref{eq:soft-msm-cm}}
\KwOut{$A\in\mathbb{R}^{m\times m}$}
\BlankLine

Initialise $A_{i,j}\leftarrow 0$ for all $(i,j)$\;
$A_{m,m}\leftarrow 1$\;

\For{$i \leftarrow m$ \KwTo $1$}{
  \For{$j \leftarrow m$ \KwTo $1$}{
    \If{$(i,j)=(m,m)$}{\textbf{continue}}
    $w \leftarrow 0$\;

    \If{$i+1\le m$}{
      $\tau_v \leftarrow \mathrm{trans}_\gamma(x_{i+1},x_i,y_j;c)$\;
      $w \leftarrow w + A_{i+1,j}\cdot \exp\!\Big(\tfrac{C_{i+1,j}-(C_{i,j}+\tau_v)}{\gamma}\Big)$\;
    }

    \If{$j+1\le m$}{
      $\tau_h \leftarrow \mathrm{trans}_\gamma(y_{j+1},y_j,x_i;c)$\;
      $w \leftarrow w + A_{i,j+1}\cdot \exp\!\Big(\tfrac{C_{i,j+1}-(C_{i,j}+\tau_h)}{\gamma}\Big)$\;
    }

    \If{$i+1\le m$ \textbf{and} $j+1\le m$}{
      $\tau_d \leftarrow (x_{i+1}-y_{j+1})^2$\;
      $w \leftarrow w + A_{i+1,j+1}\cdot \exp\!\Big(\tfrac{C_{i+1,j+1}-(C_{i,j}+\tau_d)}{\gamma}\Big)$\;
    }

    $A_{i,j}\leftarrow w$\;
  }
}
\Return $A$\;
\end{algorithm}

\subsection{Soft-MSM Jacobian}
\label{sec:softmsm-jacobian}

For most learning tasks we require derivatives with respect to the input time series itself.  These can be obtained by applying a Jacobian transformation to the alignment matrix.
We first derive the partial derivatives of the differentiable transition cost (Equation~\ref{eq:soft-msm-trans}) with respect to its scalar inputs $(x,y,z)$.
Recall that the smooth transition cost depends on the soft minimum
\(s(x,y,z)=\mathrm{softmin}_\gamma((x-y)^2,(x-z)^2)\).
Using the product rule, the partial derivatives are:

\begin{align}
\label{eq:softmsm-trans-grad-x}
\frac{\partial\,\mathrm{trans}_\gamma}{\partial x}
&= -\frac{\partial g}{\partial x}\, s
  + (1-g)\,\frac{\partial s}{\partial x}, \\[6pt]
\label{eq:softmsm-trans-grad-y}
\frac{\partial\,\mathrm{trans}_\gamma}{\partial y}
&= -\frac{\partial g}{\partial y}\, s
  + (1-g)\,\frac{\partial s}{\partial y}, \\[6pt]
\label{eq:softmsm-trans-grad-z}
\frac{\partial\,\mathrm{trans}_\gamma}{\partial z}
&= -\frac{\partial g}{\partial z}\, s
  + (1-g)\,\frac{\partial s}{\partial z}.
\end{align}

\vspace{4pt}
Equations~\ref{eq:softmsm-trans-grad-x},~\ref{eq:softmsm-trans-grad-y}, and~\ref{eq:softmsm-trans-grad-z} are implemented in Algorithm~\ref{alg:soft-msm-trans-grads}.

\begin{algorithm}[!htb]
\DontPrintSemicolon
\LinesNumbered
\caption{$\mathrm{trans\_grads}_\gamma(x,y,z;c)$}
\label{alg:soft-msm-trans-grads}
\KwIn{scalars $x,y,z$; cost $c>0$; temperature $\gamma>0$}
\KwOut{$(d_x, d_y, d_z)$ partial derivatives of $\mathrm{trans}_\gamma$ w.r.t.\ $(x,y,z)$}
\BlankLine
$\varepsilon \leftarrow 10^{-12}$ \tcp*{fixed numerical stabiliser}

$a \leftarrow x - y$\;
$b \leftarrow x - z$\;
$u \leftarrow a \cdot b$\;

$r \leftarrow \sqrt{u^2 + \varepsilon}$\;
$q \leftarrow 1/r$\;
$q_3 \leftarrow q^3$\;

$g \leftarrow \tfrac{1}{2}\!\left(1 - u\,q\right)$\;
$s \leftarrow \mathrm{softmin}_\gamma\!\bigl((x-y)^2,\,(x-z)^2\bigr)$\;

\BlankLine
\textit{// Gate derivatives:}\;
$\frac{\partial g}{\partial x} \leftarrow -\tfrac{1}{2}\!\left[(a+b)\,q - u^2(a+b)\,q_3\right]$\;
$\frac{\partial g}{\partial y} \leftarrow \phantom{-}\tfrac{1}{2}\!\left[b\,q - u^2 b\,q_3\right]$\;
$\frac{\partial g}{\partial z} \leftarrow \phantom{-}\tfrac{1}{2}\!\left[a\,q - u^2 a\,q_3\right]$\;

\BlankLine
\textit{// Softmin weights for $s$ (two-argument softmin, stabilised):}\;
$d_1 \leftarrow (x-y)^2$\;
$d_2 \leftarrow (x-z)^2$\;
$s_0 \leftarrow \min(d_1,d_2)$\;

$e_1 \leftarrow \exp(-(d_1 - s_0)/\gamma)$\;
$e_2 \leftarrow \exp(-(d_2 - s_0)/\gamma)$\;

$\pi_1 \leftarrow \frac{e_1}{e_1 + e_2}$\;
$\pi_2 \leftarrow 1 - \pi_1$\;

$\frac{\partial s}{\partial x} \leftarrow 2\big[\pi_1(x-y) + \pi_2(x-z)\big]$\;
$\frac{\partial s}{\partial y} \leftarrow -2\pi_1(x-y)$\;
$\frac{\partial s}{\partial z} \leftarrow -2\pi_2(x-z)$\;

\BlankLine
\textit{// Combine by product rule:}\;
$d_x \leftarrow -(\tfrac{\partial g}{\partial x})\, s + (1-g)\,(\tfrac{\partial s}{\partial x})$\;
$d_y \leftarrow -(\tfrac{\partial g}{\partial y})\, s + (1-g)\,(\tfrac{\partial s}{\partial y})$\;
$d_z \leftarrow -(\tfrac{\partial g}{\partial z})\, s + (1-g)\,(\tfrac{\partial s}{\partial z})$\;

\Return $(d_x, d_y, d_z)$
\end{algorithm}

With the derivatives of the transition cost defined, we can now compute the derivative of the overall Soft-MSM objective with respect to the time series $x$.
Let $F_\gamma(x, y) = C_{m,m}$ denote the Soft-MSM objective obtained from the forward recursion in Equation~\ref{eq:soft-msm-cm}.

Because each entry of the cost matrix $C$ depends recursively on the local transition costs,
the total derivative of $F_\gamma$ can be expressed as a weighted sum of these local contributions.
To accumulate these derivatives efficiently, we reuse the alignment matrix $A$ computed in Algorithm~\ref{alg:soft-msm-backward},
which encodes the expected occupancy of each alignment node.
By combining $A$ with the transition derivatives derived above,
we obtain the Jacobian of $F_\gamma$ with respect to $x$.
As in Soft-DTW, the gradient naturally decomposes into three edge types: match, vertical, and horizontal,
corresponding to the local operations in the Soft-MSM recursion.
\begin{align}
\intertext{\textit{Match edges (diagonal moves):}}
\label{eq:softmsm-g}
G_{i,j}
&=
A_{i,j}\cdot
\exp\!\Big(\tfrac{C_{i,j}-(C_{i-1,j-1}+(x_i-y_j)^2)}{\gamma}\Big),
\quad i,j\ge 2,\\[6pt]
\intertext{\textit{Vertical edges (split/merge along }$x$\textit{):}}
\label{eq:softmsm-v}
V_{i,j}
&=
A_{i,j}\cdot
\exp\!\Big(\tfrac{C_{i,j}-(C_{i-1,j}+\mathrm{trans}_\gamma(x_i,x_{i-1},y_j;c))}{\gamma}\Big),
\quad i\ge 2,\\[6pt]
\intertext{\textit{Horizontal edges (split/merge along }$y$\textit{):}}
\label{eq:softmsm-h}
H_{i,j}
&=
A_{i,j}\cdot
\exp\!\Big(\tfrac{C_{i,j}-(C_{i,j-1}+\mathrm{trans}_\gamma(y_j,y_{j-1},x_i;c))}{\gamma}\Big),
\quad j\ge 2.
\end{align}

Finally, using Equations~\ref{eq:softmsm-g},~\ref{eq:softmsm-v}, and~\ref{eq:softmsm-h} together with the local definitions in Equation~\ref{eq:soft-msm-trans},
the gradient with respect to $x_i$ can be expressed as:
\begin{align}
\label{eq:softmsm-grad}
\frac{\partial F_\gamma}{\partial x_i}
&=
\underbrace{\sum_{j=1}^{m} 2(x_i-y_j)\,G_{i,j}}_{\text{match edges}}
\;+\;
\underbrace{\sum_{j=2}^{m}
  H_{i,j}\,\partial_{x_i}\mathrm{trans}_\gamma(y_j,y_{j-1},x_i;c)
}_{\text{horizontal edges}} \notag \\[4pt]
&\quad+\;
\underbrace{
\sum_{j=1}^{m}\!\Big[
  V_{i,j}\,\partial_{x_i}\mathrm{trans}_\gamma(x_i,x_{i-1},y_j;c)
 +V_{i+1,j}\,\partial_{x_i}\mathrm{trans}_\gamma(x_{i+1},x_i,y_j;c)
\Big]
}_{\text{vertical edges}}.
\end{align}

The $i$th element of the gradient $\frac{\partial F_\gamma}{\partial \mathbf{x}}$ quantifies the influence of element $x_i$ on the total alignment cost. This gradient can then be used in downstream learning tasks such as averaging, clustering, or classification. An implementation of Soft-MSM and the associated gradient function are available in \texttt{aeon}\footnote{\url{https://github.com/aeon-toolkit/aeon}} and further examples are provided on the associated repository\footnote{\url{https://github.com/time-series-machine-learning/soft-msm}}.

\subsection{Theoretical Properties of Soft-MSM}
\label{sec:theory}
In this section we explore the smoothness,
limiting behaviour and the loss of metric structure under soft
relaxation, before introducing a
divergence-corrected formulation. We also show
the runtime complexity has the same asymptotic order as MSM.

Let $F_\gamma(x, y) = C_{m,m}$ denote the Soft-MSM objective defined by
the forward recursion in Equation~\ref{eq:soft-msm-cm}, with smoothing parameter
$\gamma > 0$ and stabilisation parameter $\epsilon > 0$.

\subsubsection{Smoothness}

\begin{proposition}[Differentiability]
For any $\gamma > 0$ and $\epsilon > 0$, the Soft-MSM objective
$F_\gamma(x, y)$ is continuously differentiable with respect to all
elements of $x$ and $y$.
\end{proposition}

\begin{proof}
The Soft-MSM recursion (Equation~\ref{eq:soft-msm-cm}) is constructed from compositions
of the following functions: squared differences $(x_i - y_j)^2$,
which are smooth; the soft minimum operator (Equation~\ref{eq:softmin}), which
is smooth for $\gamma > 0$; and the transition function
$\mathrm{trans}_\gamma(x, y, z; c)$ (Equation~\ref{eq:soft-msm-trans}),
which is smooth for $\epsilon > 0$ due to the smooth gate and soft minimum.

The boundary conditions are smooth functions of the inputs. Each entry
$C_{i,j}$ is defined recursively as a composition of smooth functions of
previous entries. By induction over $(i,j)$, all entries of $C$ are
continuously differentiable. Since $F_\gamma(x,y) = C_{m,m}$, the result
follows.
\end{proof}

\subsubsection{Limiting Behaviour}

\begin{proposition}[Hard-alignment limit]
As $\gamma \to 0^+$ and $\epsilon \to 0^+$, the Soft-MSM recursion
converges pointwise to a hard-min dynamic programming recursion with
MSM-style transitions and squared pointwise deviations. The limiting
objective differs from standard MSM in its use of squared rather than
absolute deviations, a modification required to ensure differentiability.
\end{proposition}

\begin{proof}
The soft minimum operator satisfies
\[
\lim_{\gamma \to 0^+} \text{softmin}_\gamma(a_1, \ldots, a_k) = \min(a_1, \ldots, a_k).
\]
Similarly, the smooth gate $g(x,y,z)$ converges pointwise to the
indicator function of $x$ lying between $y$ and $z$ as
$\epsilon \to 0^+$. Therefore, the transition function
$\mathrm{trans}_\gamma$ converges to a piecewise-defined transition cost
that mirrors MSM but with squared deviations in place of absolute
differences. Substituting these limits into Equation~\ref{eq:soft-msm-cm} yields a
hard-min recursion with the same move/split/merge structure as MSM, but
with squared local costs.
\end{proof}

\subsubsection{Non-metricity of soft relaxations}
\label{sec:non-metric}

\begin{proposition}[Soft-MSM is not a metric]
\label{prop:non-metric}
For any $\gamma > 0$, the Soft-MSM objective $F_\gamma$ does not satisfy
the identity of indiscernibles and therefore is not a metric.
\end{proposition}

\begin{proof}
It is sufficient to give a counter example. Consider two identical length-two
series
\[
x=y=(0,0).
\]
Then \(C_{1,1}=0\). For the boundary transitions, we have
\(u=(0-0)(0-0)=0\), so the smooth gate has value \(g=1/2\), and
\[
\mathrm{softmin}_\gamma(0,0)=-\gamma\log 2.
\]
Hence each boundary transition has cost
\[
t
=
\mathrm{trans}_\gamma(0,0,0;c)
=
c-\frac{\gamma}{2}\log 2.
\]
The final cell is therefore
\[
F_\gamma(x,x)
=
C_{2,2}
=
\mathrm{softmin}_\gamma(0,2t,2t)
=
-\gamma \log\!\left(1+2\exp(-2t/\gamma)\right).
\]
Since \(\exp(-2t/\gamma)>0\), the quantity inside the logarithm is strictly
larger than \(1\). Therefore
\[
F_\gamma(x,x)<0.
\]
Thus \(F_\gamma(x,x)\neq 0\) even though \(x=y\), so the identity of
indiscernibles and non-negativity fail. Hence Soft-MSM is not a metric.
\end{proof}

More generally, any log-sum-exp relaxation over alignment paths will fail
to satisfy the identity of indiscernibles whenever the self-alignment
partition function assigns non-zero weight to at least one path other than
the zero-cost diagonal path. The counterexample above shows that this
condition holds for Soft-MSM.

\subsubsection{Divergence Formulation}

The fact that $F_\gamma(x,x)$ can differ from zero, and can even take negative values (Proposition~\ref{prop:non-metric}),
introduces a bias that can prevent $F_\gamma$
from behaving as a meaningful dissimilarity, since pairwise comparisons are offset by input-dependent baselines.
This motivates a debiased analogue,
constructed in the same way as the Soft-DTW divergence of
\cite{blondel21softdtwdivergence}.

\begin{definition}[Soft-MSM divergence]
\label{def:soft-msm-divergence}
The Soft-MSM divergence is defined as
\[
D_\gamma(x, y) = F_\gamma(x, y) - \tfrac{1}{2} F_\gamma(x, x) - \tfrac{1}{2} F_\gamma(y, y).
\]
\end{definition}
By construction, $D_\gamma(x, x) = 0$ for all $x$. A natural further question is whether $D_\gamma(x,y) \ge 0$ for all
$x,y$, mirroring the corresponding property of the Soft-DTW divergence. Non-negativity is desirable because it allows the divergence to be interpreted
as a dissimilarity, with zero corresponding to identical inputs and larger
values indicating increasing mismatch.

For Soft-DTW, non-negativity is established by considering the partition
function, a weighted sum over all alignment paths, in which each path
$P \in \mathcal{P}$ is weighted by the exponential of its negative path
cost from Equation~\ref{eq:path-cost}:

\[
K_\gamma(x,y) \;=\; \sum_{P\in\mathcal{P}} \exp\!\big(-\mathcal{C}(P;x,y)/\gamma\big).
\]
The soft objective satisfies
\[
F_\gamma(x,y) = -\gamma \log K_\gamma(x,y).
\]
The proof for Soft-DTW shows that $K_\gamma$ is a positive semidefinite (PSD)
kernel on $\mathbb{R}^m \times \mathbb{R}^m$.
Hence the
Cauchy--Schwarz inequality gives
\[
K_\gamma(x,y)^2 \;\le\; K_\gamma(x,x)\,K_\gamma(y,y),
\]
which, after applying the logarithmic transformation above, implies
$D_\gamma(x,y) \ge 0$.

The PSD property for Soft-DTW relies on two structural facts. First, each
local cost is a pairwise function $\delta(x_i,y_j)$ of one point from each
series. Second, for squared Euclidean $\delta$,
$\exp(-\delta(x_i,y_j)/\gamma)$ is itself a PSD kernel~\cite{blondel21softdtwdivergence}.
Consequently, each path weight factorises as a product of pairwise PSD
kernel evaluations, and $K_\gamma$ can be expressed as a sum over such
path weights. Since products and sums of PSD kernels remain PSD,
$K_\gamma$ inherits positive semidefiniteness; equivalently, this is an R-convolution kernel construction~\cite{haussler99convolution}.

Soft-MSM does not satisfy the first condition. The off-diagonal transition
costs $\mathrm{trans}_\gamma(x_i,x_{i-1},y_j;c)$ and
$\mathrm{trans}_\gamma(y_j,y_{j-1},x_i;c)$ each depend on two points from
one series and one from the other. The corresponding local factors
$\exp(-\mathrm{trans}_\gamma/\gamma)$ are therefore not pairwise kernels
of the form $k(x_i,y_j)$, so the Soft-DTW factorisation argument does not
extend directly to Soft-MSM.

Thus, although the debiased form removes the self-similarity bias by
construction, non-negativity would require an additional argument beyond
the convexity of the soft-minimum operator. Establishing such a guarantee for Soft-MSM is left to future work.



\subsubsection{Computational Complexity}

\begin{proposition}
For two time series of length $m$, Soft-MSM has time complexity
$\mathcal{O}(m^2)$ and space complexity $\mathcal{O}(m^2)$, matching the
asymptotic complexity of standard MSM.
\end{proposition}

\begin{proof}
Standard MSM computes an $m \times m$ dynamic-programming matrix,
requiring constant work per cell and therefore $\mathcal{O}(m^2)$ time
and $\mathcal{O}(m^2)$ space.

Soft-MSM uses the same dynamic-programming lattice. The forward
recursion computes one $m \times m$ cost matrix, and each cell requires
a constant number of evaluations of the smooth transition function and
soft minimum. The backward recursion similarly processes each cell
once. Therefore, Soft-MSM has the same asymptotic time and space
complexity as MSM, namely $\mathcal{O}(m^2)$.
\end{proof}

\section{Experimental Evaluation}
\label{sec:results}

To evaluate Soft-MSM, we follow the experimental methodology introduced by~\cite{cuturi2017softdtw} for Soft-DTW.
We also extend their design in two ways: (i) by running each of the experiments on a larger set of datasets, and (ii) by including additional downstream evaluations for clustering and classification.
Unless stated otherwise, all experiments are configured as in~\cite{cuturi2017softdtw}.
Performance is measured on the 112 UCR datasets~\cite{dau19ucr}.

For experiments involving many estimators across many datasets, we report results using complementary summary and pairwise comparison tools.
We compare average ranks using a critical difference (CD) diagram~\citep{demsar06comparisons} with a Wilcoxon signed-rank test for pairwise comparisons, and form cliques using Holm correction as recommended by~\citep{garcia08pairwise} and~\citep{benavoli16pairwise}.
We use $\alpha = 0.1$ for all hypothesis tests.
CD diagrams (e.g., Figure~\ref{fig:centroids}) show average estimator ranks, with cliques indicated by solid bars; estimators within a clique are not significantly different.
However, CD diagrams alone can obscure effect sizes and dataset-level behaviour, so we also report summary heat maps following~\citep{ismail2023approach} and include scatter plots for direct pairwise comparisons.

All experiments were conducted using the \texttt{aeon} open-source time series machine learning toolkit\footnote{\url{https://github.com/aeon-toolkit/aeon}} and evaluated using the \texttt{tsml-eval} package.
We additionally provide Numba, PyTorch, and TensorFlow implementations of Soft-MSM, together with a reproducibility notebook that documents how to run all experiments\footnote{\url{https://github.com/time-series-machine-learning/soft-msm}}.
All results reported in this paper are also included as CSV files (one per evaluation metric) alongside the notebook.

\subsection{Averaging}
In order to assess Soft-MSM, we focus on the problem of time series averaging and its application in clustering and classification. To assess averaging in isolation, we follow the experimental design used to assess Soft-DTW~\cite{cuturi2017softdtw}:
for each dataset, we select a class label at random and then sample $10$ time series from that class. We then compute a barycentre using three averaging procedures for both DTW and MSM:
\begin{itemize}
    \item DTW-based methods: DBA~\cite{petitjean11dba}, SSG-DBA~\cite{schultz18ssgba}, and Soft-DTW barycentres (Soft-DBA)~\cite{cuturi2017softdtw};
    \item MSM-based methods: MBA~\cite{holder23mba}, a stochastic subgradient MSM barycentre (SSG-MBA), and our proposed Soft-MSM barycentre (Soft-MBA).
\end{itemize}
For each dataset and procedure we repeat the experiment $10$ times with different random seeds. To assess the quality of the barycentres we always evaluate under the hard elastic distance that defines the geometry of interest. For DTW-based methods we report the DTW Fr\'echet loss (Equation~\ref{eq:frechet} with $d=\text{DTW}$), while for MSM-based methods we report the MSM Fréchet loss (Equation~\ref{eq:frechet} with $d=\text{MSM}$).
This ensures that Soft-MBA is compared fairly against MBA and SSG-MBA in the MSM geometry, and Soft-DBA against DBA and SSG-DBA in the DTW geometry. For soft methods we consider $\gamma \in \{1, 0.1, 0.01, 0.001\}$.

\begin{table}[t]
\caption{Percentage of datasets on which the Soft-DTW Barycentre Average achieves a lower DTW loss compared to Dynamic Barycentre Averaging (DBA) and Stochastic Subgradient Dynamic Barycentre Averaging (SSG-DBA).}
\label{tab:soft_dtw_vs_baselines_by_gamma}
\centering
\makebox[\textwidth][c]{%
\begin{tabular}{l|c|c}
\toprule
\multirow{2}{*}{$\gamma$} \\
& \textbf{DBA} & \textbf{SSG-DBA} \\
& Better (\%) & Better (\%) \\
\midrule
1   & 2.8\,\% & 14.7\,\% \\
0.1   & 48.6\,\% & 49.5\,\% \\
0.01   & 86.2\,\% & 78.9\,\% \\
0.001   & 93.6\,\% & 83.5\,\% \\
\bottomrule
\end{tabular}
}
\end{table}

\begin{table}[t]
\caption{Percentage of datasets on which the Soft-MSM Barycentre Average achieves a lower MSM loss compared to MSM Barycentre Averaging (MBA) and Stochastic Subgradient MSM Barycentre Averaging (SSG-MBA).}
\label{tab:soft_msm_vs_baselines_by_gamma}
\centering
\makebox[\textwidth][c]{%
\begin{tabular}{l|c|c}
\toprule
\multirow{2}{*}{$\gamma$} \\
& \textbf{MBA} & \textbf{SSG-MBA} \\
& Better (\%) & Better (\%) \\
\midrule
1   & 7.3\,\% & 26.6\,\% \\
0.1   & 94.5\,\% & 96.3\,\% \\
0.01   & 97.2\,\% & 100.0\,\% \\
0.001   & 97.2\,\% & 100.0\,\% \\
\bottomrule
\end{tabular}
}
\end{table}

Table~\ref{tab:soft_dtw_vs_baselines_by_gamma} summarises the percentage of datasets on which the Soft-DTW barycentre achieves a lower DTW loss than DBA and SSG-DBA for different values of~$\gamma$.
As in~\cite{cuturi2017softdtw}, we observe that as $\gamma$ decreases, Soft-DTW increasingly outperforms both DBA and the subgradient method, achieving lower DTW loss on $93.6\%$ and $83.5\%$ of datasets respectively for $\gamma = 0.001$.
These results are in line with the findings in~\cite{cuturi2017softdtw} even with the expanded set of datasets used.

Table~\ref{tab:soft_msm_vs_baselines_by_gamma} reports the analogous comparison for MSM-based averaging, where we evaluate all methods under MSM loss.
Here, Soft-MBA exhibits an even stronger advantage.
For moderate to low values of $\gamma$ (e.g., $\gamma = 0.1$ and $\gamma = 0.01$), Soft-MBA achieves lower MSM loss than MBA and SSG-MBA on almost all datasets, and matches or exceeds SSG-MBA on all datasets for the smallest $\gamma$ considered.
In contrast, for $\gamma = 1$ the soft objective is overly smoothed and the performance gap is smaller, mirroring the behaviour of Soft-DTW.

Overall, the results in Tables~\ref{tab:soft_dtw_vs_baselines_by_gamma} and~\ref{tab:soft_msm_vs_baselines_by_gamma} demonstrate that Soft-MBA consistently improves upon non-differentiable MSM-based averaging schemes under MSM loss. They also show that,  in the MSM geometry, Soft-MBA attains larger gains over its hard counterpart than Soft-DTW does over DBA in the DTW geometry.
These results indicate that making MSM differentiable yields substantial benefits for averaging, particularly when combined with a suitably chosen smoothing parameter $\gamma$.

\subsubsection{Qualitative Analysis: Class Prototypes}

To complement the quantitative evaluation, we consider whether the resulting
barycentres provide interpretable class prototypes. We use the
\textsc{CricketX}\footnote{\url{https://timeseriesclassification.com/description.php?Dataset=CricketX}}
dataset, for which the class structure has a clear physical interpretation.
The data consist of motion recordings of an umpire making cricket hand
signals~\cite{ko02cricket}; \textsc{CricketX} contains the X-axis measurements
only, with the left- and right-hand signals concatenated. In cricket, the
official umpire signal for a six is to raise both arms above the head. This
produces a characteristic movement in both hands, although there is substantial
variation between trials. Figure~\ref{fig:cricket} shows the prototypes
obtained by the three MSM-based averaging methods and Soft-DTW. Soft-MSM
clearly recovers a peak for each hand while maintaining a relatively smooth
trajectory elsewhere. The alternative methods recover this structure less
clearly and produce more variable prototypes.
\begin{figure}[htb]
    \includegraphics[width=1.0\linewidth]{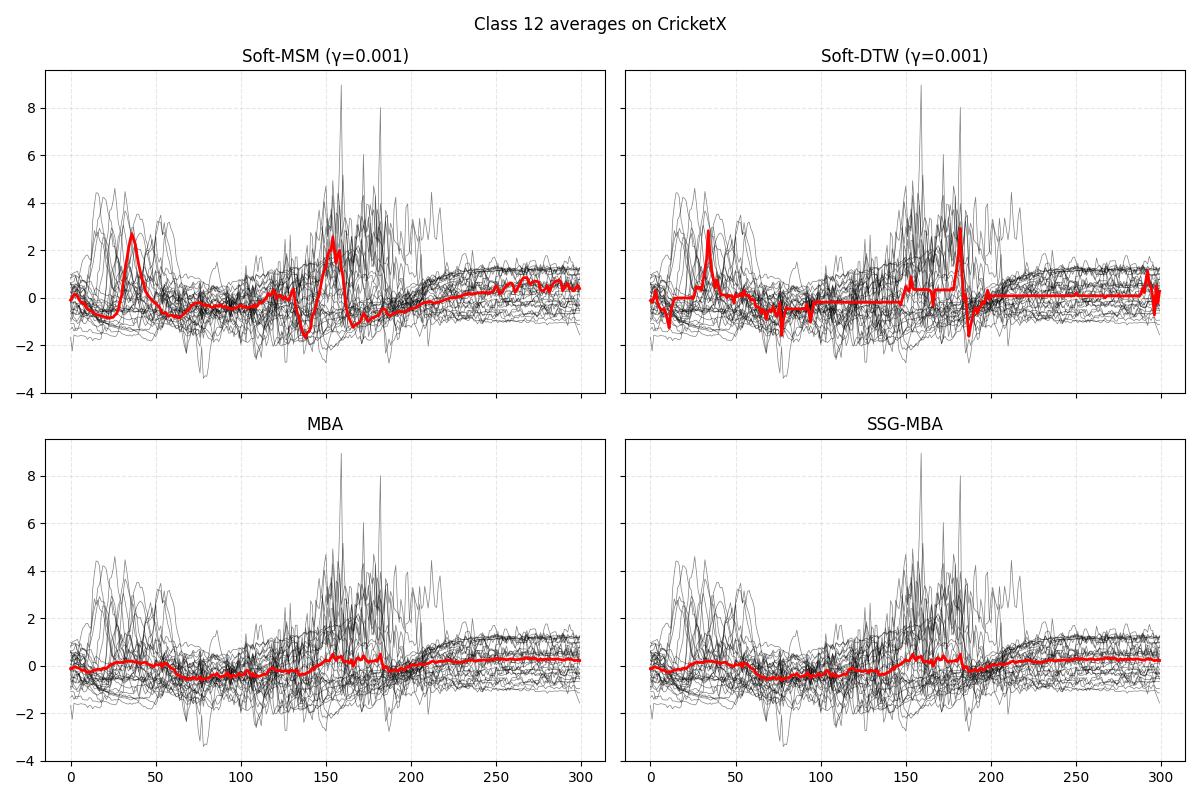}
    \caption{Example prototypes formed for class 12 (six runs) of the \textsc{CricketX} dataset with four alternative averaging algorithms.}
    \label{fig:cricket}
\end{figure}
To directly compare Soft-DTW to Soft-MSM, we assess both distances on two common applications that employ elastic distances: clustering and classification.

\FloatBarrier
\subsection{Time Series Clustering}
\label{sec:results-clust}

The most common application of averaging in time series analysis is within the $k$-means clustering algorithm. For each dataset, the number of clusters $k$ is set equal to the number of classes. We then run $k$-means using both Soft-DTW barycentre averaging (Soft-DBA) and Soft-MSM barycentre averaging (Soft-MBA) to compute the centroids.

Following the experimental design of~\cite{cuturi2017softdtw}, we use the combined train--test set and evaluate four values of the smoothing parameter $\gamma$: $1.0$, $0.1$, $0.01$, and $0.001$. Table~\ref{tab:soft_hard_dist_vs_baselines_by_gamma} reports the percentage of datasets for which $k$-means with Soft-MBA and Soft-DBA achieves a lower MSM and DTW loss, respectively, compared to their hard counterparts (measured by sum of squared errors).

\begin{table}[h!]
\caption{Percentage of datasets in which Soft-MBA and Soft-DBA based $k$-means achieve lower MSM and DTW loss, respectively (sum of squared error).}
\label{tab:soft_hard_dist_vs_baselines_by_gamma}
\centering
\begin{tabular}{c|c|c}
\toprule
$\gamma$ & \textbf{MBA} & \textbf{DBA} \\
         & Better (\%)  & Better (\%)  \\
\midrule
1.0   & 8.3\,\%  & 1.9\,\%  \\
0.1   & 87.2\,\% & 10.6\,\% \\
0.01  & 84.8\,\% & 55.4\,\% \\
0.001 & 83.5\,\% & 84.8\,\% \\
\bottomrule
\end{tabular}
\end{table}

The results in Table~\ref{tab:soft_hard_dist_vs_baselines_by_gamma} highlight a clear and consistent advantage of Soft-MSM for barycentre estimation. Across all values of $\gamma$, Soft-MBA improves upon its hard counterpart in the large majority of datasets, with performance remaining stable as the smoothing parameter varies. In contrast, the behaviour of Soft-DBA is considerably more sensitive to $\gamma$: while it performs poorly for large smoothing values, its advantage over hard DBA only becomes consistent for very small $\gamma$. This suggests that Soft-MSM provides a more robust relaxation of MSM than Soft-DTW does for DTW in the context of barycentre averaging.

For context, we extend our experimental evaluation to include comparisons with commonly used time series clustering approaches. In particular, we compare both methods against standard $k$-means with Euclidean averaging (denoted $k$-AVG) and $k$-Shape~\cite{paparrizos16kshapes}, which has recently been shown to be a strong benchmark for time series clustering ~\cite{paparrizos25tsclustcomp}.

\begin{figure}[htb]
    \centering
    \begin{tabular}{c c}
        \includegraphics[width=0.5\linewidth]{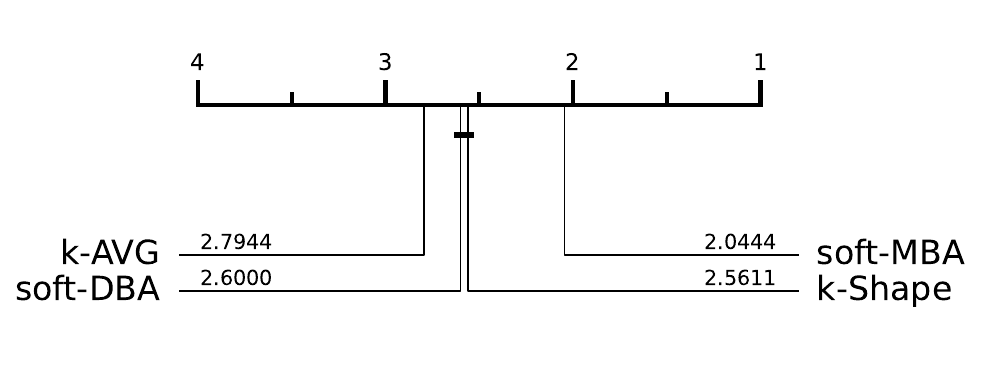} &
        \includegraphics[width=0.5\linewidth]{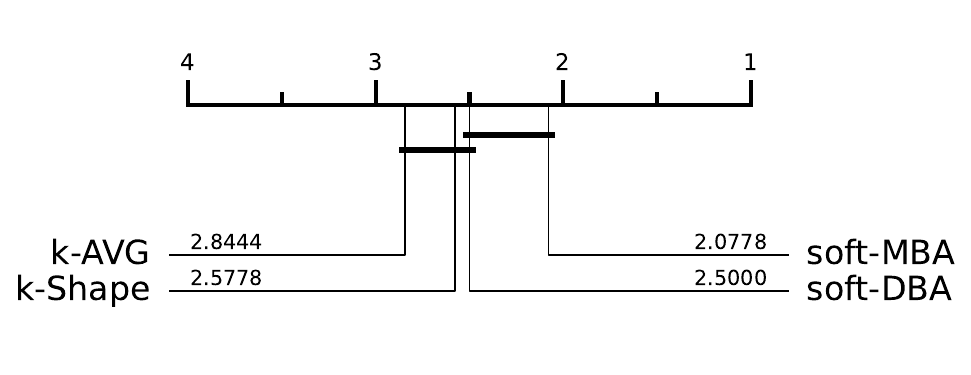} \\
        Clustering Accuracy & Adjusted Rand Index \\
    \end{tabular}
    \caption{Average ranks of four clustering algorithms.}
    \label{fig:cd}
\end{figure}

\begin{figure}[htb]
    \centering
    \includegraphics[width=1.0\linewidth]{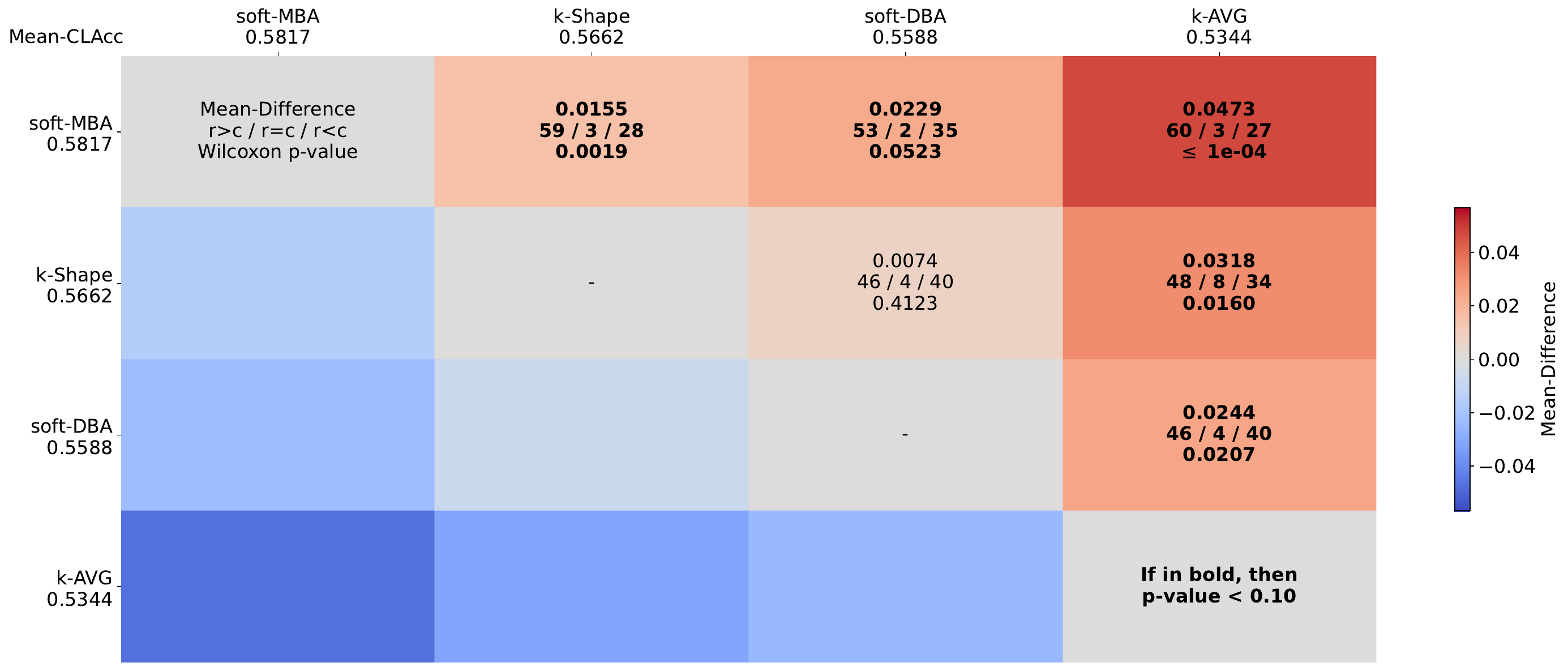}
    \caption{Summary performance of four clustering algorithms on the UCR datasets.}
    \label{fig:heatmap}
\end{figure}

Figure~\ref{fig:cd} shows the average ranks for clustering accuracy (CLAcc) and Adjusted Rand Index (ARI) across the four methods. Figure~\ref{fig:heatmap} provides further detail on  pairwise relative performance. Soft-MBA is significantly better than $k$-Shape in both CLAcc and ARI and significantly better than Soft-DBA in terms of CLAcc.

These results show that Soft-MBA is an effective differentiable MSM-based clustering objective. Soft-MBA consistently outperforms Soft-DBA, with the difference being particularly pronounced for clustering accuracy, where it is over 2\% more accurate on average. Its performance is also competitive with $k$-Shape, despite the two methods relying on fundamentally different notions of similarity and optimisation.

\subsection{Time Series Classification}

The nearest-centroid classifier~\cite{hastie2009elements} provides a simple yet effective alternative to the traditional $k$-nearest neighbour ($k$-NN) approach for time series classification~\cite{veenman11sparse}. It addresses two main limitations of $k$-NN. First, the $k$-NN classifier must store all training instances for use at inference time. Second, to predict a class label, it requires computing the distance between the test series and every series in the training set. In contrast, the nearest-centroid method represents each class by a single prototype. As a result, only one representative time series per class needs to be stored, and at prediction time, only as many distance computations as there are class labels are required to determine the predicted class.

Following the experiments in~\cite{cuturi2017softdtw}, we construct one prototype per class by computing a barycentre of the training instances in that class using both Soft-DTW and Soft-MSM. We are not advocating either approach as a recommended algorithm for time series classification (TSC): the accuracy using a nearest centroid with a single prototype per class is not competitive with modern TSC algorithms. Rather, we use the nearest centroid classifiers to compare the relative effectiveness of averaging. We use two variants: the first uses the soft function for forming centroids and also for finding the neighbour. We call these Soft-DBA and Soft-MBA. The second version follows the approach in~\cite{cuturi2017softdtw} by averaging with the soft function but using the standard distance for classification. These are designated Soft-DBA2 and Soft-MBA2. Figure~\ref{fig:centroids} shows the average ranks of these four algorithms for accuracy and balanced accuracy over the 112 UCR datasets. A solid bar indicates there is no significant difference between the algorithms covered (using Wilcoxon signed rank test with Holm correction for multiple testing). The figure shows that both MBA variants are significantly better than the DBA versions.

\begin{figure}[htb]
    \centering
            \begin{tabular}{c c}
    \includegraphics[width=0.5\linewidth]{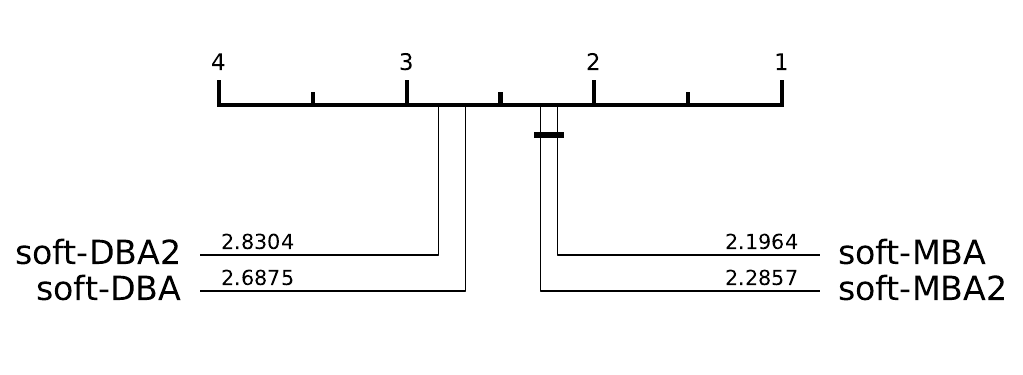} &
    \includegraphics[width=0.5\linewidth]{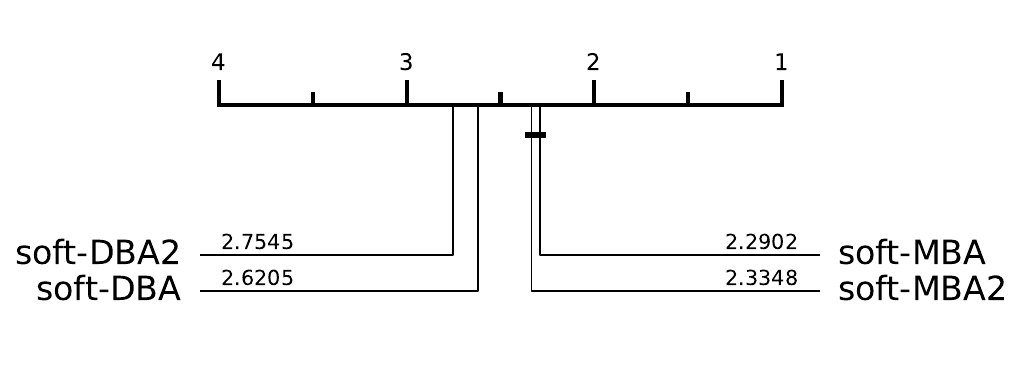}\\
    Accuracy & Balanced Accuracy \\
    \end{tabular}
    \caption{Critical difference diagrams of test data performance for four classification algorithms on 112 UCR datasets.}
    \label{fig:centroids}
\end{figure}

\begin{figure}[htb]
    \centering
            \begin{tabular}{c c}
    \includegraphics[width=0.5\linewidth]{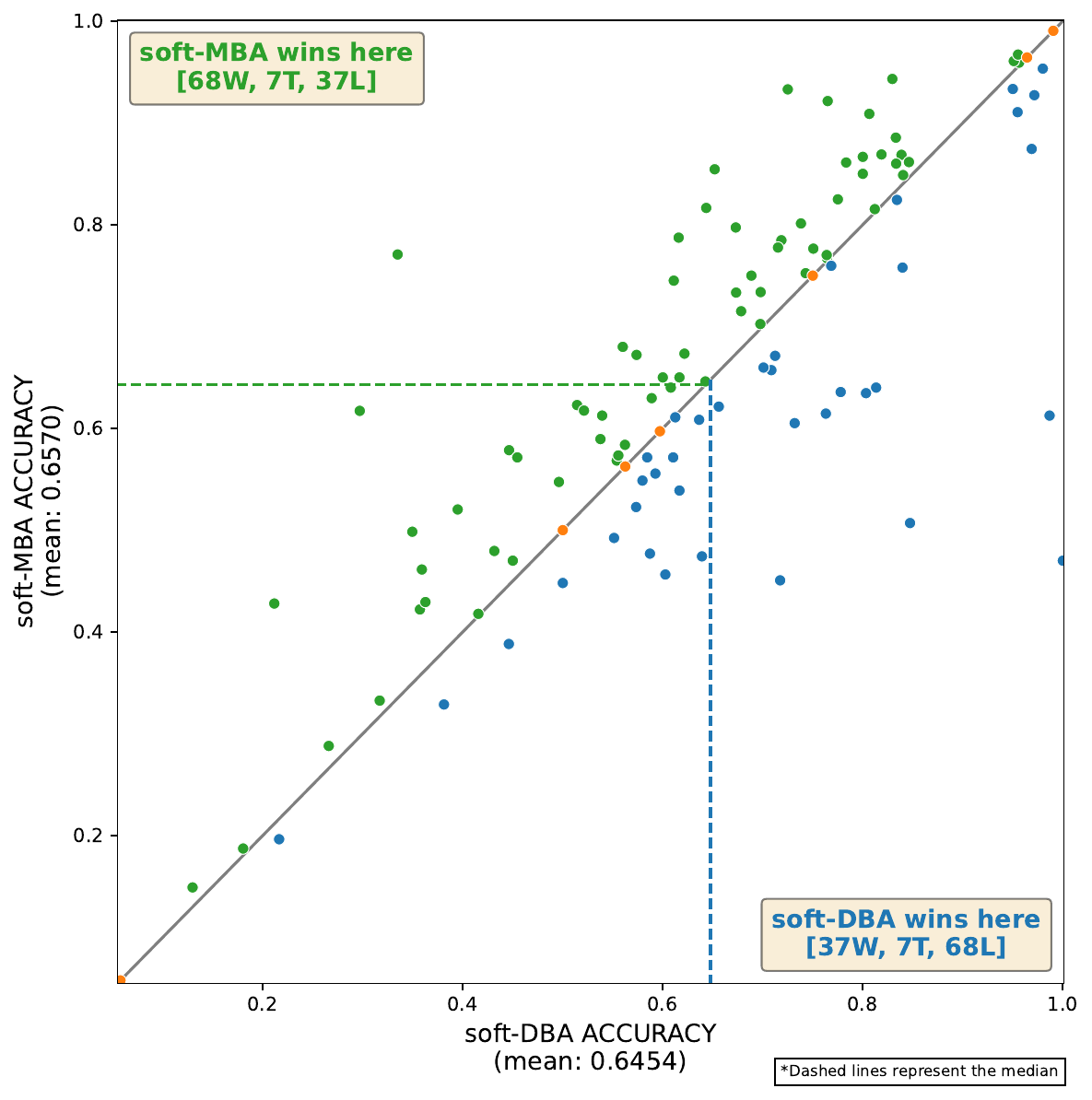} &
    \includegraphics[width=0.5\linewidth]{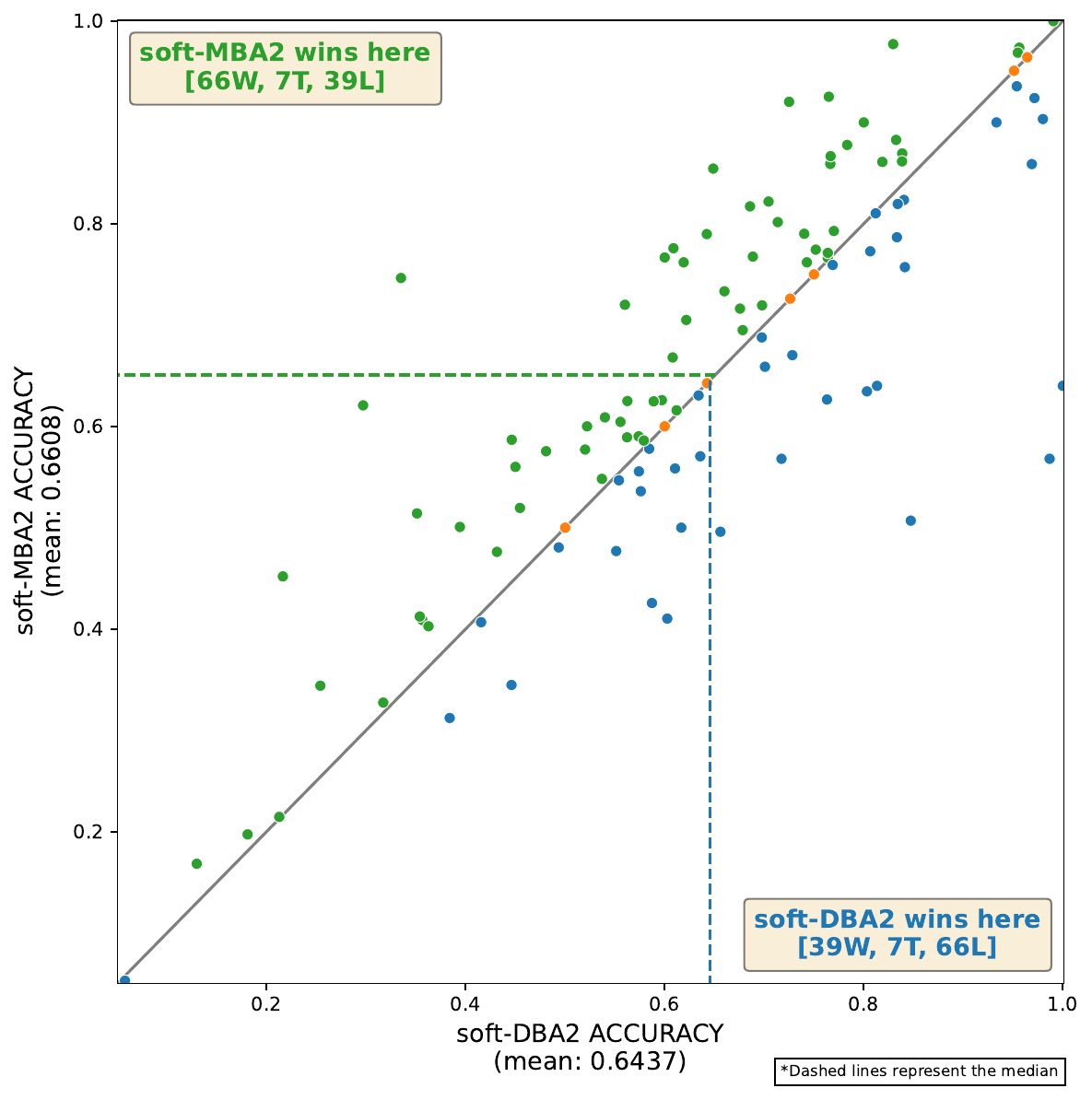}\\
    Variant 1 & Variant 2 \\
    \end{tabular}
    \caption{Scatter plots of accuracy for two versions of k-centroids. The first variant (left) uses the soft distance for forming centroids and for finding neighbours. The second (right) uses the standard distance function for finding neighbours.}
    \label{fig:scatter}
\end{figure}
Figure~\ref{fig:scatter} shows scatter plots of the pairwise performance of the two variants.


\subsection{Limitations}

Soft-MSM has several limitations. First, the soft relaxation trades the exact
metric property of MSM for differentiability, so Soft-MSM should be treated as
an alignment-based loss rather than a metric distance. Second, as with MSM and
Soft-DTW, the dynamic-programming formulation requires full \(m \times m\)
cost and alignment matrices. The asymptotic complexity remains
\(\mathcal{O}(m^2)\), but Soft-MSM introduces larger constant factors through
the soft minimum, transition gate, and backward pass required for gradients.

Third, performance depends on the smoothing parameter \(\gamma\). Very large
values over-smooth the alignment objective, while very small values approach
the hard recursion and can reduce the benefit of smoothing. The stabilisation
parameter \(\epsilon\) also affects the smooth gate, although in our experiments
we fix it to a small numerical constant. Finally, the unnormalised Soft-MSM
objective inherits the self-similarity bias of other log-sum-exp relaxations:
in general \(F_\gamma(x,x)\neq 0\). A divergence-corrected version can remove
this bias, but requires three Soft-MSM evaluations.

\section{Conclusion}
\label{sec:conclusions}

We have introduced Soft-MSM, a differentiable alignment-based loss inspired by the Move-Split-Merge distance. Soft-MSM retains MSM's context-aware transition structure while trading exact metricity for differentiability. We derived the
forward and backward recursions, the associated soft alignment matrix, and the gradient with respect to the input series. The method is implemented in the open-source \texttt{aeon} toolkit to support reuse and reproducibility.

Our experiments show that this relaxation is effective for time series averaging. Soft-MSM barycentre averaging achieves lower MSM Fr\'echet loss than
MBA and SSG-MBA on the vast majority of datasets, and the resulting prototypes
lead to improved clustering and nearest-centroid classification relative to
Soft-DTW-based alternatives. The \textsc{CricketX} case study further illustrates
that Soft-MSM can produce smoother and more interpretable class prototypes.

We further demonstrated that the improved MSM barycentres translate into better clustering performance. Using $k$-means with elastic centroids on the UCR archive, Soft-MBA achieves significantly higher accuracy than Soft-DBA on more datasets and is competitive with a strong baseline such as $k$-Shape. For nearest-centroid classification, Soft-MSM provides a simple way to construct MSM-based prototypes in a fully differentiable manner.

Our empirical evaluation has focused primarily on time series averaging and its downstream use in clustering and nearest-centroid classification. More broadly, Soft-MSM can be used as a differentiable alignment-based loss in settings where
Soft-DTW is currently used as a loss or similarity measure. Potential applications include deep time series forecasting~\cite{le19shape,qing22forecasting,cortez24dayaheadpv}, weakly supervised alignment and segmentation, generative models for data
augmentation~\cite{kamycki20spawner}, class rebalancing for imbalanced
problems~\cite{qiu26esmote}, and similarity-based representation learning with
models such as Series2Vec~\cite{foumani24series2vec}. In these settings,
Soft-MSM provides an alternative inductive bias based on move, split, and merge
operations, which may be particularly useful for piecewise-constant or
event-driven series.

Future work will focus on extending the empirical evaluation to multivariate
and unequal-length time series, exploring scalable approximations for long
sequences, and experimenting with using Soft-MSM in modern deep learning architectures.
Overall, Soft-MSM provides a useful alternative to Soft-DTW for
differentiable, alignment-aware learning with time series.

\noindent\textbf{Reproducibility.}
To support FAIR (Findable, Accessible, Interoperable, and Reusable) principles in research, Soft-MSM is available in the \texttt{aeon} toolkit, alongside an extensive range of optimised time series distance functions. Code to reproduce our experiments and spreadsheets of the results used to generate graphs are available on the associated repository\footnote{\url{https://github.com/time-series-machine-learning/soft-msm}}.

\section*{Acknowledgements}

This work has been supported by the UK Research and Innovation Engineering and Physical Sciences Research Council (grant reference EP/W030756/2). The authors acknowledge the use of the IRIDIS High Performance Computing Facility and associated support services at the University of Southampton, in the completion of this work. We would like to thank all those responsible for helping maintain the time series classification archives and those contributing to open-source implementations of the algorithms.

\bibliographystyle{unsrtnat}
\bibliography{TSCMaster,sn-bibliography}

\end{document}